\documentclass[a4paper,12pt]{article}

\usepackage[truedimen,margin=20truemm]{geometry}

\usepackage{amsmath,amsthm,comment,amssymb}
\usepackage{graphicx}
\usepackage{color}

\newcommand{\rE}{\mathrm{E}}

\newcommand{\Var}{\mathrm{Var}}

\newcommand{\cM}{\mathcal{M}}

\newcommand{\cP}{\mathcal{P}}

\newcommand{\cF}{\mathcal{F}}
\newcommand{\cG}{\mathcal{G}}
\newcommand{\cH}{\mathcal{H}}

\newcommand{\cN}{\mathcal{N}}

\newcommand{\cX}{\mathcal{X}}

\newcommand{\bR}{\mathbf{R}}

\newcommand{\argmin}{\mathop{\rm argmin}\limits}

\usepackage{algorithm}
\usepackage{algorithmic}
\usepackage{enumitem}

\usepackage[T1]{fontenc}
\usepackage[utf8]{inputenc}
\usepackage[font=small,labelfont=bf,tableposition=top]{caption}

\theoremstyle{plain}
\newtheorem{thm}{Theorem}
\newtheorem{lem}[thm]{Lemma}
\newtheorem{prop}[thm]{Proposition}
\newtheorem{cor}[thm]{Corollary}

\theoremstyle{definition}
\newtheorem{df}[thm]{Definition}

\newtheorem{assump}[thm]{Assumption}

\newtheorem{rmk}[thm]{Remark}

\usepackage[utf8]{inputenc} 
\usepackage[T1]{fontenc}    
\usepackage{hyperref}       
\usepackage{url}            
\usepackage{booktabs}       
\usepackage{amsfonts}       
\usepackage{nicefrac}       
\usepackage{microtype}      

\title{Empirical Hypothesis Space Reduction}

\author{%
  Akihiro Yabe\\
  NEC corporation\\
  \texttt{a-yabe@cq.jp.nec.com} \\
  \and
  Takanori Maehara\\
  RIKEN AIP\\
  \texttt{takanori.maehara@riken.jp} \\
}

\date{}

\begin{document}

\maketitle

\begin{abstract}
	Selecting appropriate regularization coefficients is critical to performance 
	with respect to regularized empirical risk minimization problems. 
	Existing theoretical approaches attempt to determine the coefficients in order for regularized empirical objectives to be upper-bounds of true objectives, uniformly over a hypothesis space.
	Such an approach is, however, known to be over-conservative, especially in high-dimensional settings with large hypothesis space. 
	In fact, an existing generalization error bound in variance-based regularization is $O(\sqrt{d \log n/n})$, where $d$ is the dimension of hypothesis space, and thus the number of samples required for convergence linearly increases with respect to $d$.
	This paper proposes an algorithm that calculates regularization coefficient, one which results in faster convergence of generalization error $O(\sqrt{\log n/n})$ and whose leading term is independent of the dimension $d$.
	This faster convergence without dependence on the size of the hypothesis space is achieved by means of empirical hypothesis space reduction, which, with high probability, successfully reduces a hypothesis space without losing the true optimum solution.
	Calculation of uniform upper bounds over reduced spaces, then, 
	enables acceleration of the convergence of generalization error.
\end{abstract} 

\section{Introduction}
Regularization is a standard method for improving generalization
by means of penalizing risky hypotheses.
This paper considers the following regularized empirical risk minimization problem:
\begin{align}
	h_{\lambda}(x^n) := \argmin_{h \in \cH} L(h; x^n) + \frac{\lambda}{\sqrt{n}} r(h), \label{intro1}
\end{align}
where $\cH$ is a hypothesis space, $L(\cdot; x^n)$ is an empirical risk function determined by $n$ i.i.d. samples $x^n = (x_1, \dots, x_n)$, $\lambda \in \bR_+$ is a regularization scale, and $r: \cH \to \bR_+$ is a regularizer.
Examples of regularizers include $\ell_p$-regularizers ($p\geq 0$) for penalizing large norms~\cite{kakade2009complexity,shalev2014understanding}
and variance-based regularizers for penalizing high variances~\cite{maurer2009empirical,namkoong2017variance}.
With a suitable choice of $\lambda$ and $r$,
we can improve the convergence of generalization error $L^*(h_{\lambda}(x^n)) - L^*_{\min}$,
where $L^*: \cH \to \bR$ is the true risk function and $L^*_{\min}$ is its minimum.

When a regularizer $r$ is fixed, performance depends solely on its coefficient $\lambda$; hence, this must be carefully determined.
In the context of variance-based regularization,
Maurer and Pontil~\cite{maurer2009empirical} showed the following generalization bounds:
Let $\cM(n)$ be the covering number of $\cH$, a detailed definition of which is given in the subsequent section.
Given any probability $\delta >0$,
defining  $\lambda = \sqrt{\log(\cM(n)/\delta)}$
results in a solution $h_{\lambda}$ in \eqref{intro1}
that satisfies the following bounds for generalization error with a probability of at least $1-\delta$:
\begin{align*}
	L^*(h_{\lambda}(X^n)) - L^*_{\min} \leq \sqrt{ \frac{32 V^*\log(\cM(n)/\delta )}{n}} +  O\left( \frac{ \log \cM(n)}{n} \right).
\end{align*}
Here, $V^*$ is the variance of the loss function with the true optimum hypothesis.
They defined the scale $\lambda$ in order for regularized empirical objective (i.e., RHS of \eqref{intro1})
to become an upper-bound of the true objective $L^*$ uniformly over $h \in \cH$, with a probability of $1-\delta$.
Minimization of this probabilistic upper-bound 
contributes to decreasing the bounded true objective,
and thus the resulting empirical hypothesis $h_{\lambda}(X^n)$ is guaranteed
with respect to the true optimum $L^*_{\min}$
with the same probability.
Such a scale $\lambda$ for uniform-bounding over $\cH$ is required to be proportional to the \emph{size} of the hypothesis space $\cH$, or the logarithm of the covering number $\cM$.
These criteria thus make it possible to control the scale $\lambda$
by means of confidence probability $\delta$.

Unfortunately, such theoretical criteria for determining regularization scale is
known to be impractical, especially in high dimensional setting.
Roughly speaking, the above $\lambda$ and the resulting generalization bound
are proportional to $\log\cM(n)$, and
if $\cH$ is $d$-dimensional space,
then $\log\cM(n) \geq d \log n$.
This implies that the number of samples $n$ required for the convergence
is linearly dependent on the dimension of $\cH$.
%
%
%
\paragraph{Our contribution}
We propose an algorithm for calculating a regularization scale that results in faster convergence of generalization error.
Our algorithm consists of two parts:
the first is  \emph{empirical hypothesis space reduction},
in which, with high probability, the hypothesis space is reduced through the use of empirical samples, without loss of the true optimum solution.
The second part is calculation of uniform bounds on the basis of reduced space.
Since the reduced space is asymptotically singleton (assuming the uniqueness of the optimum hypothesis),
our algorithm achieves dimensional-free convergence of generalization error.
In particular, for the variance-based regularizer, assuming locally  quadratic true risk $L^*$,
the hypothesis $h(X^n)$ calculated by our algorithm achieves the following generalization error:
\begin{align*}
	L^*(h(X^n)) - L^*_{\min} \leq \sqrt{ \frac{8 V^* \log(c n/\delta)}{ n}} + O\left( \sqrt{\frac{\log n}{n^{1 + 2/d}} } \right).
\end{align*}
Here $c$ is a constant which is independent of $n$, and $d$ is the dimension of the hypothesis space $\cH$.
Note that the coefficient of the dominant $O(\log n/\sqrt{n})$ term is independent of the size of hypothesis space $\cH$. 
Our algorithm can be applied to any regularizer, such as the $\ell_p$ regularizer or the variance-based regularizer, 
and any construction of uniform bounds, e.g., based on VC dimension~\cite{vapnik1971uniform}, Rademacher complexity~\cite{bartlett2002rademacher}, or covering number~\cite{maurer2009empirical}.
Our algorithm can thus accelerate the convergence of the generalization error for a very general
class of regularized empirical risk minimization problems.
\paragraph{Related studies}
Reduction of hypothesis space for speeding up convergence
has previously been proposed~\cite{boucheron2005theory,shapiro2009lectures,bartlett2005local,koltchinskii2006local},
and the most relevant study can be found in the context of variance-based regularization.
Namkoong and Duchi~\cite{namkoong2017variance} extended the idea of~\cite{maurer2009empirical}
by using the technique of distributionally robust optimization~\cite{ben2013robust,bertsimas2017robust},
and proposed several uniform bounds 
on the basis of covering number, VC dimension~\cite{vapnik1971uniform}, and Rademacher complexity~\cite{bartlett2002rademacher}.
For tighter construction of such uniform bounds,
\cite[Theorem 4]{namkoong2017variance} presents
the calculation of Rademacher complexity
on the basis of restricted hypothesis space.
One technical difference is that
we conduct restriction on the basis of
a regularized empirical solution,
while \cite{namkoong2017variance} (and related techniques in the study of local Rademacher complexity~\cite{bartlett2005local,koltchinskii2006local}) have done so on the basis of
a non-regularized empirical solution.
Our technique makes possible simpler and more unified analysis with fewer assumptions,
which makes it applicable to arbitrary regularization settings.

\section{Preliminary}\label{secPre}
\subsection{Risk minimization problem}
Let $\cH$ be a hypothetical space, $\cX \subseteq \bR^m$ be a sample space, and $\ell: \cH \times \cX \to \bR_+$ be a loss function.
Let $\cP$ be a distribution over $\cX$, and $L^*$ be a risk function defined by $L^*(h):= \rE_{X \sim \cP} \left[ \ell(h, X) \right] $.
Our goal is to find a hypothesis that minimize the risk function $L^*$:
\begin{align}
	L^*_{\min} := \min_{h \in \cH} L^*(h). \label{trueSol}
\end{align}
The true distribution $\cP$, however, is rarely available in practice.
Thus, we here assume that we have $n$ i.i.d. samples $x^n = (x_1, \dots, x_n)$ from $\cP$.
Our aim is to create an algorithm
that with high probability outputs optimized hypothesis $h(X^n)$
with small generalization error $L^*(h(X^n)) -L^*_{\min}$ in sample distribution $X^n \sim \cP^n$.

Hereafter, 
we tacitly assume that $n$ is an integer satisfying $n \geq 6$.
We represent random variable drawn from $\cP$ by upper case $X$,
and an element of $\cX$ by lower case $x$.
\subsection{Existing study: variance-based regularization}\label{subsecExist}
This section introduces the generalization error bound proven by \cite{maurer2009empirical} for variance-based regularization.
Let us first introduce the problem setting in variance-based regularization.
We here assume that the value range of the loss function $\ell$ is $[0,1]$.
Given samples $x^n = (x_1,x_2,\dots,x_n) \in \cX^n$, let us define empirical risk function $L(h ; x^n)$ by
\begin{align*}
	L(h; x^n) := \frac{1}{n} \sum_{i=1}^n \ell(h, x_i).
\end{align*}
For each $h \in \cH$, let us define the true variance $V^*(h)$ and empirical variance $V_n(h;x^n)$ of loss $\ell(h,\cdot)$ by
\begin{align*}
	V^*(h) &:= \rE_{X \sim \cP} [ (\ell(h,X) - L^*(h))^2 ], \\
	V_n(h ; x^n) &:= \frac{1}{n(n-1)} \sum_{1 \leq i<j \leq n} (\ell(h,x_i) - \ell(h,x_j))^2 .
\end{align*}
For a regularization scale $\lambda \in \bR_+$, the empirical solution $h_{\lambda}(x^n)$ in the study of variance-based regularization is defined as follows:
\begin{align}
	h_{\lambda}(x^n) := \argmin_{h \in \cH} L(h; x^n) + \frac{\lambda}{\sqrt{n}} \sqrt{V_n(h;x^n)}. \label{dfVBR}
\end{align}
For this setting, the following error bound is proven by~\cite{maurer2009empirical}:
For $\varepsilon_n > 0$ and $x^{2n} \in \cX^{2n}$,
let us define $\cM’(\varepsilon, x^{2n})$ as the minimum cardinality $|\cH_{0}|$ of $\cH_{0} \subseteq \cH$ satisfying the following property:
for all $h \in \cH$,
there exists $h_0 \in \cH_{0}$ satisfying $|\ell(h,x_i) - \ell(h_0, x_i)| \leq \varepsilon$ for all $i=1,2,\dots,2n $.
We then introduce the covering number $\cM(n)$ as $\cM(n) := 30 \max_{x^{2n} \in \cX^{2n}} \cM’(1/n, x^{2n})$.
Let us denote the true minimizer of $L^*$ by $h^* \in \cH$,
and the following generalization error bound then holds.
\begin{thm}[{\cite[Theorem 15]{maurer2009empirical}}] \label{thmPrev}
	For $\delta \in (0,1)$ and $\lambda = \sqrt{18 \log (\cM(n)/\delta)}$,
	the optimized hypothesis \eqref{dfVBR} satisfies the following bound with a probability of at least $1-\delta$ in sample distribution $X^n \sim \cP^n$:
	\begin{align*}
		L^*(h_{\lambda}(X^n)) - L^*_{\min}  \leq \sqrt{\frac{32V^*(h^*) \log(\cM(n)/\delta)}{n}} + \frac{22  \log(\cM(n)/\delta)}{n}.
	\end{align*}
\end{thm}
The growth rate of $\cM(n)$ in $n$ is polynomial in many cases~\cite{maurer2009empirical},
and it is known that $\log \cM(n) = O(\log^{3/2} n)$
for the bounded linear functionals in the reproducing kernel Hilbert space associated with Gaussian kernels~\cite{guo2002covering}.
If the hypothesis space $\cH$ is embedded in a  real space $\bR^d$,
then, typically, the term $\log \cM(n)$ is linearly dependent on $d$.
Thus, the size of the term $\log \cM(n)$ can be understood as $\log \cM(n) \gtrapprox d \log n$.

\section{Main Results}\label{secReduction}
\subsection{Terminology}
This section introduces terminology that we employ. 
Any object with a $*$ mark is intended to be unknown to 
the algorithm that we wish to create.
Let $r^*: \cH \to \bR_+$ denote an ideal but unknown regularizer,
and $r_n: \cH \times \cX^N \to \bR_+$ denote its empirical estimates.
A typical example of $r^*(h)$ and $r_n(h;x^n)$ are the square-root of
the true variance $\sqrt{V^*(h)}$ and the empirical variance $\sqrt{V_n(h;x^n) }$ of the loss $\ell$, respectively.
For a regularizer without uncertainty, such as $\ell_1$-regularizer and $\ell_2$-regularizer,
we have $r^* = r_n$.
We define the notion of the accuracy of an estimator $r_n$ as follows.
\begin{df}\label{df1}
	A pair $(r_n, \Delta_n)$ is referred to as a \emph{guaranteed empirical regularizer} (with respect to $r^*$)
	if the following inequality bound holds with probability at least $1 - \delta/N$ in sample distribution $X^n \sim \cP^n$:
	\begin{align}
		|r^*(h) - r_n(h;X^n)| \leq \Delta_n, \quad \forall h \in \cH. \label{dfEmpReg}
	\end{align}
\end{df}

For a regularizer $r_n$,
ideally, our algorithm would calculate a maximum, $\max_{h \in \cG} r_n(h; x^n)$,
over non-convex subspace $\cG \subseteq \cH$.
Such a calculation would in general, however, be computationally intractable.
As it is quite reasonable to assume that we can calculate 
some upper-bound $u_n(\cG)$ of $\max_{h \in \cG} r_n(h; x^n)$,
we can then impose some consistency on $u_n$, including monotonicity with respect to $\cG$, as follows:
\begin{df}\label{df2}
	A function $u_n: 2^{\cH} \times \cX \to \bR_+ \cup \{\infty \}$ is referred to as an \emph{empirical regularization upper-bound} if
	the following holds for any $x^n \in \cX^n$ and $\cF \subseteq \cG \subseteq \cH$ with $\sup_{h \in \cG} r_n(h;x^n) < \infty$:
	\begin{align*}
		\sup_{h \in \cG} r_n(h;x^n) &\leq u_n(\cG; x^n) < \infty ,\\
		u_n(\cF; x^n) &\leq u_n(\cG; x^n).
	\end{align*}
	We refer to $u^*_n: 2^{\cH} \to \bR_+$ as a \emph{true regularization upper-bound} if for any $x^n$ and $\cG \subseteq \cH$,
	it holds that 
	\begin{align}
		u_n (\cG;x^n) \leq u_n^*(\cG) < \infty.\label{dfU}
	\end{align}
\end{df}
Note that, for our proof, it is sufficient to require \eqref{dfU} for $x^n $ satisfying \eqref{dfEmpReg}.
Next, we define the notion of a uniform bound,
which is a standard notion that has been utilized for bounding generalization error in previous studies~\cite{maurer2009empirical,namkoong2017variance}.
\begin{df}\label{df3}
	A pair $(\alpha_n, \beta_n) \in \bR_+^2$ of values is referred to as a \emph{uniform bound} if the following holds with a probability of at least $1 - \delta/n$ in $X^n \in \cP^n$:
	\begin{align}
		|L^*(h) - L(h; X^n) | \leq \frac{\alpha_n }{\sqrt{n}} r^*(h) + \beta_n, \quad \forall h \in \cH. \label{dfUnif}
	\end{align}
\end{df}
Let us next propose a novel generalization of the uniform bound, referred to as a \emph{spatial uniform bound},
for calculating a uniform bound over reduced subspace $\cF \subseteq \cH$.
\begin{df}\label{df4}
	A pair of functions $(\mu_n, \nu^*_n)$ where $\mu_n, \nu_n^* : 2^{\cH} \to \bR_+$ is referred to as a \emph{spatial uniform bound} if 
	the following two properties hold:
	
	(i) For any $\cF \subseteq \cG \subseteq \cH$, it holds that $\mu_n(\cF) \leq \mu_n(\cG) \leq \alpha_n$ and $\nu_n^*(\cF) \leq \nu_n^*(\cG)$.
	
	(ii) For any $\cF \subseteq \cH$, the following holds with a probability at least $1 - (n-2)\delta/n $ in $X^n \sim \cP^n$:
	\begin{align}
		|L^*(h) - L(h; X^n) | \leq  \frac{\mu_n(\cF)}{\sqrt{n}}r^*(h) + \nu_n^*(\cF), \quad \forall h \in \cF. \label{dfSpatial}
	\end{align}
\end{df}
Condition (i) requires monotonicity.
Note that the condition $\mu_n(\cH) \leq \alpha_n$, which implies $\mu_n(\cG) \leq \alpha_n$ for any $\cG \subseteq \cH$, can be naturally satisfied
since the required confidence level $1- (n-2)\delta/n$ for $\mu_n(\cH)$ is less than that $1- \delta/n$ for $\alpha_n$ when $n \geq 6$.
Condition (ii) is a generalization of uniform bounding~\eqref{dfUnif} for a subspace $\cF \subseteq \cH$.
\subsection{Example of parameterization in variance-based regularization}\label{subsecVBR}
This section provides concrete examples of functions and parameters that satisfy the conditions in Definitions~\ref{df1}–\ref{df4}.
We specify parameters for variance-based regularization, assuming the following conditions.
\begin{assump}\label{assumpLip}
	(i) $\cH$ is a bounded subset of $\bR^d$, and $\| \cdot \|$ denotes its Euclidean norm.
	
	(ii) $\ell$ is defined over $\bR^d \times \cX$, 
	and the value range of $\ell$ is $[0,1]$. In other words, $\ell: \bR^d \times \cX \to [0,1]$.
	
	(iii) The Lipschitz constant $c_{\ell}$ of $\ell$, which satisfies $|\ell(h_1, x) - \ell(h_2, x)| \leq c_{\ell} \|h_1 - h_2  \|$ for any $x \in \cX$ and $h_1,h_2 \in \cH$, is known.
	
	(iv) For any $x \in \cX$, $\ell(\cdot, x)$ is twice differentiable over $\bR^d$. In addition, there exists $p_1^*, p_2^* \in \bR_+$ that satisfies, for any $x \in \cX$ and $h$ in the convex hull of $\cH$, $|\partial \ell(\cdot ,x) / \partial h_i | \leq p^*_1$ and $\|\nabla^2 \ell(\cdot ,x) \|_2 \leq p^*_2$. Here, $\|\cdot  \|_2$ is the induced norm of $\bR^{d \times d}$.
\end{assump}
In response to the notation in the previous section for general settings,
specific examples in this section are  accompanied by a superscript ${}^V$.
In variance-based regularization,
the ideal regularizer $r^{V*}$ is the square-root of variance of loss function, and the empirical regularizer $r^V_n$ is its estimate:
\begin{align*}
	r^{V*}(h) &:= \sqrt{V^*(h)},\\
	r_n^V(h;x^n) &:= \sqrt{V_n(h; x^n)}.
\end{align*}
We here introduce 
another definition for covering number $\cN$, in contrast to $\cM$ as defined in Section~\ref{subsecExist}, as follows.
For $\varepsilon >0$ and $\cF \subseteq \cH$, we define covering number $\cN(\varepsilon, \cF)$ as the minimum
cardinality $|\cF_0|$ of subset $\cF_0 \subseteq \cH$ satisfying the following property:
for any $h \in \cF$, there exists $h_0 \in \cF_0$ such that $\|h - h_0\|\leq \varepsilon$.
We then define $\Delta_n^V$ by
\begin{align*}
	\Delta_n^V := \sqrt{\frac{3\log (2n \cN(1/n,\cH)/\delta)}{n} } + \frac{4\sqrt{2} c_{\ell}}{n}.
\end{align*}
We define empirical and true regularization upper-bounds $u^V_n$ and $u^{V*}_n$, respectively, as 
trivial upper-bounds: for any $\cF \subseteq \cH$ and $x^n \in \cX^n$,
\begin{align*}
	u^V_N(\cF;x^n) = u^{V*}_n = 1.
\end{align*}
We define a uniform bound $(\alpha^V_n, \beta^V_n)$, on the basis of Bennett’s inequality and the Lipschitz continuity, as
\begin{align*}
	\alpha^V_n := \sqrt{2 \log (2n \cN(1/n,\cH)/\delta) },\\
	\beta^V_n := \frac{(4 c_{\ell} + 1) \log (2n \cN(1/n,\cH)/\delta)  }{n}.
\end{align*}
The construction of spatial uniform bound $(\mu^V_n,\nu^{V*}_n)$
also relies on Bennett’s inequality and the Lipschitz continuity,
but we adopt the description below, which is tighter than $(\alpha^V_n, \beta^V_n)$ 
owing to the fact that $\nu^{V^*}_n$ can be unknown to our algorithm.
Let us define the local Lipschitz constant $c_{L^*}(\cF)$ of $L^*$ in $\cF$ as a minimum value satisfying $|L^*(h_1) - L^*(h_2)| \leq c_{L^*}(\cF) \|h_1 - h_2\|$ for any $h_1,h_2 \in \cF$.
We then define $\varepsilon_n > 0$ and $(\mu^{V}_n, \nu^{V*}_n)$ by
\begin{align*}
	\varepsilon_n &:= \frac{\log^{1/4} (n/\delta)}{n^{1/4 + 1/d}}, \\
	\mu^V_n(\cF) &:= \sqrt{2 \log \left( \frac{2 n\cN(\varepsilon_n,\cF)}{(n-3) \delta} \right)},\\
	\nu^{V^*}_n (\cF) &:=  \frac{2 c(\cF) \log^{1/4} (n/\delta)}{n^{1/4+1/d}} + \frac{p_2^* \log^{1/2}(n/\delta)}{n^{1/2 + 2/d}} 
	+ \frac{4\sqrt{ p_1^{*2} d^2 + c_{\ell}^2} \log^{3/4} (2dn \cN(\varepsilon_n,\cF)/\delta)  }{n^{3/4 + 1/d}} \\
	&\quad + \frac{\log (4\cN(\varepsilon_n,\cF)/\delta)}{3n}.
\end{align*}
Let us emphasize that,
although calculation of $c_{L^*}$
requires the true risk function $L^*$,
our algorithm does not refer to $\nu^{V^*}_n$ and thus to $c_{L^*}$.

The following statement guarantees that
the examples given above satisfy the desired properties in Definition~\ref{df1}--\ref{df4}.
\begin{prop}\label{propVBRCondition}
	(i) $(r_n^V, \Delta_n)$ is a guaranteed empirical regularizer with respect to $r^{V*}$.
	
	(ii) $u_n^V$ and $u_n^{V*}$ are empirical and true regularization upper-bound, respectively.
	
	(iii) $(\alpha_n^V, \beta_n^V)$ is a uniform bound.
	
	(iv) $(\mu^V_n , \nu^{V*}_n)$ is a spatial uniform bound.
\end{prop}
\subsection{Empirical hypothesis space reduction algorithm} \label{subsecAlgo}
Given a guaranteed empirical regularizer $(r_n,\Delta_n)$, an empirical regularization upper-bound $u_n$, a uniform bound $(\alpha_n,\beta_n)$, and $\mu_n$ of a spatial uniform bound $(\mu_n, \nu_n^*)$, 
Algorithm~\ref{algo1} calculates optimized hypothesis $h(x^n)$ from empirical sample $x^n \in \cX^n$ as follows.
In Line~\ref{alg1line1}, the algorithm first calculates optimum value $v(x_n)$ of the following regularized empirical risk minimization problem on the basis of a uniform bound $(\alpha_n,\beta_n)$:
\begin{align}
	v(x_n) := \min_{h \in \cH} L(h; x_n) + \frac{\alpha_n}{\sqrt{n}} r_n(h ;x^n).     \label{dfHatEta}
\end{align}
In Line~\ref{alg1line2}, the algorithm defines the subspace $\cG(x^n) \subseteq \cH$ by
\begin{align}
	\cG(x^n) :=  \left\{ h \in \cH \left| L(h;x^n) \leq v(x_n) + \frac{3 \alpha_n r(h;x^n)  + 7\alpha_n \Delta_n }{\sqrt{n}}  + 5\beta_n    \right\}  \right. \label{dfG},
\end{align}
and it then calculates empirical regularization upper-bound $u_n(\cG(x^n); x^n)$.
In Line~\ref{alg1line3}, the algorithm conducts \emph{empirical hypothesis reduction}, by reducing $\cH$ to its subspace $\cF(x^n) \subseteq \cH$ defined by by
\begin{align}
	\cF(x^n) := \left\{ h \in \cH \left| L(h; x^n) \leq v(x^n) +  \frac{ 3\alpha_n u_n(\cG(x^n)) + 5\alpha_n \Delta_n) }{\sqrt{n}}  + 5\beta_n  \right\}  \right. . \label{dfF}
\end{align}
It then calculates the spatial uniform bound $\mu_n(\cF(x^n))$ on the basis of the reduced subspace $\cF(x^n)$.
In Line~\ref{alg1line4}, the algorithm calculates the optimized hypothesis $h(x^n)$ on the basis of $\mu_N(\cF(x^n))$:
\begin{align}
	h(x^n) := \argmin_{h \in \cH} L(h; x^n) + \frac{\mu_n(\cF(x^n))}{\sqrt{n}} r_n(h; x^n). \label{dfHatH}
\end{align}
The remark below explains the computational tractability of the proposed algorithm.
\begin{rmk}
	Lines~\ref{alg1line1} and~\ref{alg1line4} calculate standard regularized empirical risk minimization.
	Although risk minimization can be non-convex (convexity is extensively studied, for example, in~\cite{namkoong2017variance}),
	this paper focuses mainly on sample complexity and thus assumes tractability.
	
	In Line~\ref{alg1line2}, the upper-bound $u_n$ of empirical regularizer $r_n$ over $\cG(x^n)$ is calculated.
	If $r_n$ is a convex function such as $\ell_2$ regularizer,
	then $\cG(x^n)$ is non-convex in general.
	Thus, exact maximization of a convex function $r_n$ over non-convex space $\cG(x^n)$ is computationally intractable in general.
	We avoid this intractability by compromising with any upper-bound $u_n$ of $r_n$.
	
	In Line~\ref{alg1line3}, the uniform bound $\mu_n(\cF(x^n))$ over $\cF(x^n)$ is calculated.
	Observe that $\cF(x^n)$ is defined by bounding $L(\cdot; x^n)$ by a constant.
	Thus, if $\cH$ is a convex subset of a vector space and the empirical risk function $L(\cdot; x^n)$ is convex, the restricted space $\cF(x^n)$ is also convex.
	We therefore suppose that the calculation of $\mu_n$ over $\cF(x^n)$ is as easy as 
	the calculation of a uniform bound $\alpha_n$ over the original space $\cH$,
	which commonly has been assumed in previous studies~\cite{maurer2009empirical,namkoong2017variance}.
\end{rmk}

\begin{algorithm}[t]
	\caption{Optimization of hypothesis with empirical hypothesis space reduction}\label{algo1}
	\begin{algorithmic}[1]
		\REQUIRE Samples $x^n \in \cX^n$
		\ENSURE Optimized hypothesis $h(x^n) \in \cH$
		\STATE Calculate optimum value $v(x_n)$ defined by \eqref{dfHatEta} \label{alg1line1}
		\STATE Define $\cG(x^n)$ by \eqref{dfG} and calculate $u_n(\cG(x^n), x^n)$ \label{alg1line2}
		\STATE Define $\cF(x^n)$ by \eqref{dfF} and calculate $\mu_n(\cF(x^n))$ \label{alg1line3}
		\STATE Optimize $h(x^n)$ by \eqref{dfHatH} \label{alg1line4}
	\end{algorithmic}
\end{algorithm}

\subsection{Theoretical analysis regarding generalization error}
Let us denote the set of true minimizer by $\cH^* := {\argmin}_{h \in \cH} L^*(h)$, and let $r^*_{\cH^*} := \min_{h^* \in \cH^*} r^*(h^*)$.
The generalization error of the output of Algorithm~\ref{algo1} will then be bounded as expressed below; this is our main theoretical result.
\begin{thm}\label{thmMain}
	The output $h(X^n)$ of Algorithm~\ref{algo1} satisfies the following bound with a probability of at least $1-\delta$ in $X^n \sim \cP^n$:
	\begin{align*}
		L^*(h(X^n)) - L^*_{\min} \leq \frac{2\mu_n(\overline{\cF}_n)}{\sqrt{n}} (r^*_{\cH^*} + \Delta_n )  + 2 \nu_n^* (\overline{\cF}_n), 
	\end{align*}
	where
	\begin{align*}
		\overline{\cF}_n := \left\{ h \in \cH \left| L^*(h) \leq L^*_{\min} +  \frac{ \alpha_n }{\sqrt{n}} (  5 u_n^*(\overline{\cG}_n) + 6 \Delta_n ) + 7\beta_n \right\} \right. ,\\
		\overline{\cG}_n := \left\{ h \in \cH \left| L^*(h) \leq L^*_{\min} +  \frac{\alpha_n }{\sqrt{n}}(6 r^*(h) + 11 \Delta_n)  + 7\beta_n\right\} \right..
	\end{align*}
\end{thm}
%
Observe that $\overline{\cF}_n$ asymptotically converges to $\cH^*$ regardless of $u^*$ and $\overline{\cG}$.
The following corollary then simplifies Theorem~\ref{thmMain}
for the asymptotic limit.
Let us define $\cH(\xi)$ for $\xi > 0$, $\mu^*_n$, and $\nu^*_n$ as
\begin{align*}
	\cH(\xi) := \{ h \in \cH \mid L(h) - L^*_{\min} \leq \xi \}.
\end{align*}
\begin{cor}\label{cor1}
	Suppose that $\limsup_{n \to \infty} u^*_n(\cH) < \infty$, $\limsup_{n \to \infty} \alpha_n < \infty$, $\lim_{n \to \infty}  \beta_n = \lim_{n \to \infty} \Delta_n = 0$, and $\lim_{n \to \infty} \sqrt{n} \nu_n(\overline{\cF}_n) =0$.
	If $\mu_n^* \in \bR_+$ for $n=1,2,\dots$ satisfy $\lim_{\xi \to 0} \limsup_{n  \to \infty } \mu_n(\cH(\xi)) / \mu_n^* = 1$, then the following bound holds with a probability of at least $1-\delta$ in $X^n \sim \cP^n$:
	\begin{align}
		L^*(h(X^n)) - L^*_{\min} = \frac{ 2\mu_n^*}{\sqrt{n}} r^*_{\cH^*} + o\left( \frac{1}{\sqrt{n}} \right). \label{corBound1}
	\end{align}
\end{cor}
The coefficient $\mu^*_n$ can be understood as the (approximately) minimum coefficient
that satisfies $| L^*(h^*) - L(h^*; X^n)| \leq  \mu^*_n r^*(h^*)/\sqrt{n}$ for all $h^* \in \cH^*$ with probability $1-\delta$.
This corollary thus shows that 
the coefficient of the leading $O(1/\sqrt{n})$ term of the generalization error
is entirely determined by the local constant $\mu^*_n r^*_{\cH^*}$,
which is independent of the size of the hypothesis space $\cH$.

In the previous study, 
Maurer and Pontil~\cite{maurer2009empirical} observed that
the bound in Theorem~\ref{thmPrev}
quickly converges if the variance on the true optimum hypothesis $V^*(h^*)$ is small.
The advantage of our bound in Theorem~\ref{thmMain} is that
it takes the convergence of the neighborhood $\overline{\cF}_n$ to the true optimal hypothesis $h^*$ into account,
which convergence is quick if the upper bound $u^*(\overline{\cG}_n)$ of the regularizer $r^*$ over the neighborhood $\overline{\cG}_n$ of $h^*$ is small.
We can thus observe that
the proposed bound in Theorem~\ref{thmMain} quickly converges
if $r^*_{\cH^*}$ is small, 
$r^*$ is uniformly small around $\cH^*$,
and the upper bound $u^*_n$ is tight.

Let us next demonstrate a concrete example that achieve the above faster convergence rate
in the context of the variance-based regularization introduced in Section~\ref{subsecVBR}.
We say that $L^*$ is \emph{locally quadratic} if the true minimizer $h^*$ of $L^*$ is unique and there exists $\gamma_0 \in (0,1]$, $a \in \bR_+$, and $b \in \bR_+$ satisfying the following condition:
For any $0 \leq \gamma \leq \gamma_0$ and $h_1,h_2 \in \cH(\gamma)$,
it holds that
\begin{align}
	L^*(h_1) - L^* &\geq a \|h_1 - h^*\|^2, \label{localQuad1} \\
	|L^*(h_1) - L^*(h_2)| &\leq b \sqrt{\gamma} \|h_1 - h_2\|. \label{localQuad2}
\end{align}
We refer to \eqref{localQuad2} as quadratic condition in the following sense:
Assuming $L^*(h_1) - L^*(h^*) \approx a \|h_1 - h^*\|^2$,
we have $| \nabla L^*(h_1) | \approx 2a \|h_1 - h^* \| \approx 2 \sqrt{a (L^*(h_1) - L^*(h^*))} \leq  2 \sqrt{a \gamma}$,
which implies \eqref{localQuad2} with $b = 2 \sqrt{a}$. 
For such a $\gamma_0$, let us define $n_0$ as a minimum integer that satisfies 
\begin{align*}
	\sqrt{ \frac{ 50 \log (2n \cN(1/n_0,\cH)/\delta) }{n_0}} + \frac{(52 c_{\ell} + 25 )\log (2n_0 \cN(1/n_0,\cH)/\delta)  }{n_0
	} \leq \gamma_0.
\end{align*}
We then define $c'$ and $c$ by
\begin{align*}
	c' &:= \sup_{n \geq 1} \left(\sqrt{ \frac{50 \log (2n \cN(1/n,\cH)/\delta)}{\log(n/\delta) } }  + \frac{(52 c_{\ell} + 25 )\log (2n \cN(1/n,\cH)/\delta)  }{\sqrt{n\log(n/\delta)}} \right), \\
	c &:=  2 \left( 4dc' /a  \right)^{d/2}.
\end{align*}
Note that such a finite constant $c'$ must exist since $\cH$ is bounded and $\log (2n \cN(1/n,\cH) ) = O(\log n)$.
\begin{cor}\label{cor2}
	Suppose that Assumption~\ref{assumpLip} holds and $L^*$ is locally quadratic.
	Suppose that we run Algorithm~\ref{algo1}
	with $r_n^V$, $\Delta^V_n$, $u^V_N$, $\alpha^V_n$, $\beta^V_n$, and $\mu^V_n,$ defined in Section~\ref{subsecVBR}.
	Then, for any $n \geq n_0$, the following bound with a probability of at least $1-\delta$ in $X^n \sim \cP^n$:
	\begin{align*}
		L^*(h(X^n)) - L^*_{\min} \leq \sqrt{\frac{8V(h^*) \log(cn /\delta )}{n}} + O\left( \sqrt{\frac{\log n}{ n^{1 + 2/d}}} \right).
	\end{align*}
\end{cor}
Under the locally quadratic condition \eqref{localQuad1} and \eqref{localQuad2},
the $o(1/\sqrt{n})$ term in Corollary~\ref{cor1} is thus specified as $O( \log^{1/2} n / n^{1/2 + 1/d} )$
for the variance-based regularization.

\section{Experiments}\label{secExp}
We demonstrate the greater efficiency of the proposed algorithms
in simple experiments with synthesis data,
whose experimental setting is introduced in the previous study~\cite{maurer2009empirical}.

\subsection{Experimental setting}
We first define $K = 500$ and $B = 1/4$.
For all $k=1,2,\dots,K$, we then generate parameters $a_k$ from the uniform distribution over $[B, 1-B]$,
and $b_k$ independently from the uniform distribution over $[0,B]$.
We then define empirical risk minimization problem as follows. 
We define $\cX = [0,1]^K$, $\cH = \{h \in \{0,1 \}^K \mid \sum_{k=1}^K h_k = 1 \}$,
and $\ell(h,x) = \sum_{k=1}^K h_k x_k$.
The distribution $\cP$ is then defined by: for each $k=1,2,\dots,K$, $X_k$ is $a_k + b_k$ or $a_k - b_k$ with equal probability $1/2$.
Note that it then holds that $\rE[X_k] = a_k$ and $\Var[X_k] = b_k$,
and the true optimum hypothesis $h^*$ is defined by $h^*_{k^*} = 1$ (and $h_k = 0$ if $k \neq k^*$), where $k^* = {\argmin}_{k} a_k$.

For this setting, given $n$ samples from $\cP$, we apply 
(non-regularized) empirical risk minimization (ERM),
variance-based regularization
with a regularization scale given by previous study~\cite{maurer2009empirical} (VBR),
and the regularization on the bases of the proposed empirical hypothesis space reduction algorithm (HSR).
More concretely, we define $\delta = 0.5$, which corresponds to upper-bounding median,
and then the regularization scale for VBR is defined by $\lambda_n = \sqrt{2\log(2K/\delta) / n}$
on the basis of~\cite[Corollary 7]{maurer2009empirical}.
For HSR, we define a series of parameters 
as $\Delta_n = \sqrt{2\log(2Kn/\delta)/(n-1)}$, $\alpha_n = \sqrt{2\log(2Kn/\delta)}$, $\beta_n = \log(2Kn/\delta) / (3n)$, and $\mu_n(\cF) = \sqrt{2\log(2n|\cF|/\delta(n-2))}$ for a finite subset $\cF \subseteq \cH$,
on the basis of Bennett's inequality (see~\cite[Theorem 3]{maurer2009empirical}) and the union bound.
Note that, for finite hypothesis space $\cH$,
$\alpha_n$ and $\mu_n$ can be rather simply defined using a concentration inequality and the union bound,
compared to the general (possibly continuous) setting introduced in Section~\ref{subsecVBR}.

The sample sizes $n$ ranged from $20$ to $2000$.
All results are average of $1000$ generations of $a_k$ and $b_k$.

\begin{figure*}[t]
	\begin{minipage}[t]{0.485\hsize}
		\begin{center}
			\includegraphics[width=\hsize]{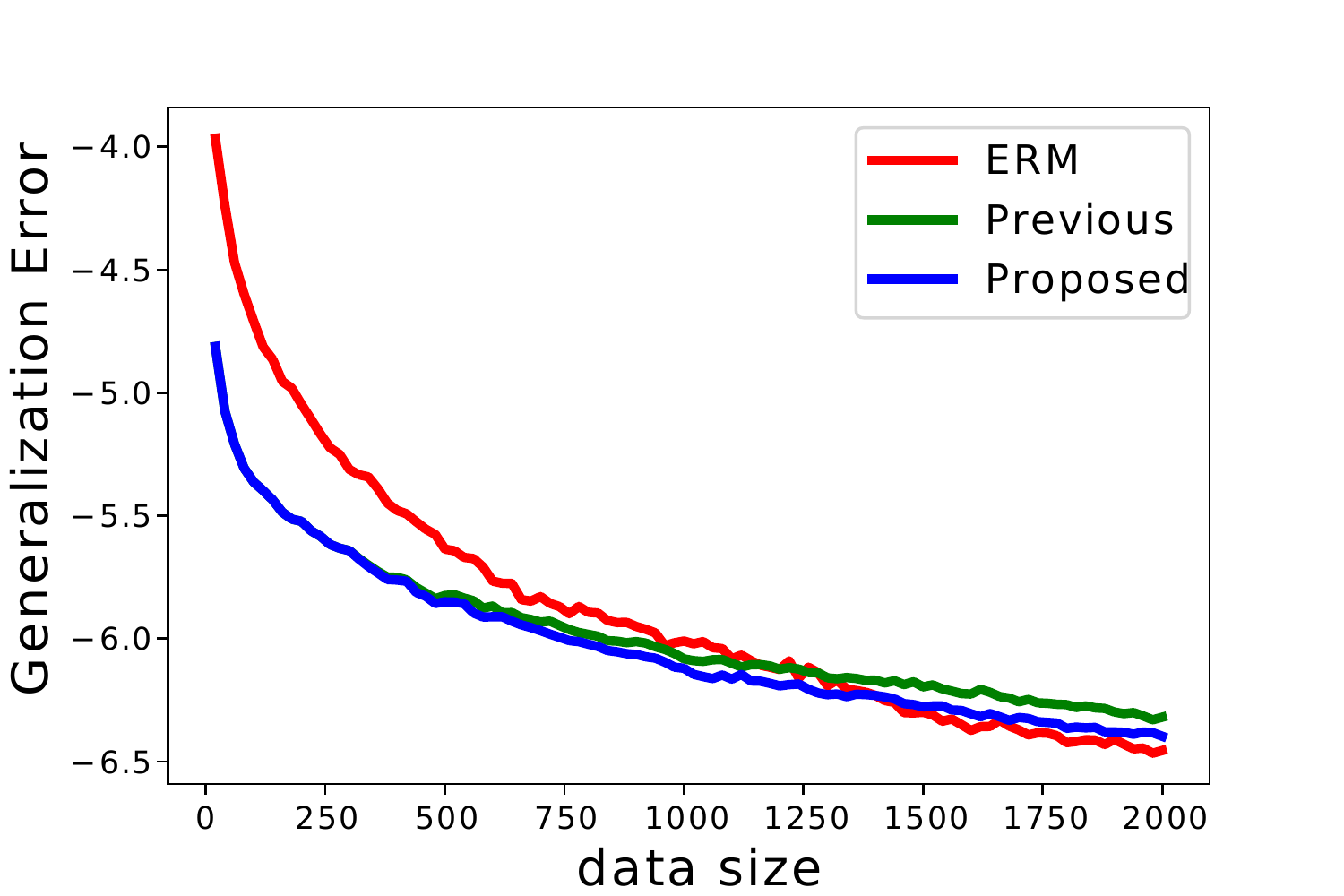}
			\caption{Convergence of generalization error.
				The horizontal line shows the number of samples, and the vertical line shows the logarithm of the generalization error. Respective red, green, and blue lines show the result of ERM, VBR, and HSR.} \label{figGenError}
		\end{center}
	\end{minipage}
	\hspace{0.02\hsize}
	\begin{minipage}[t]{0.485\hsize}
		\begin{center}
			\includegraphics[width=\hsize]{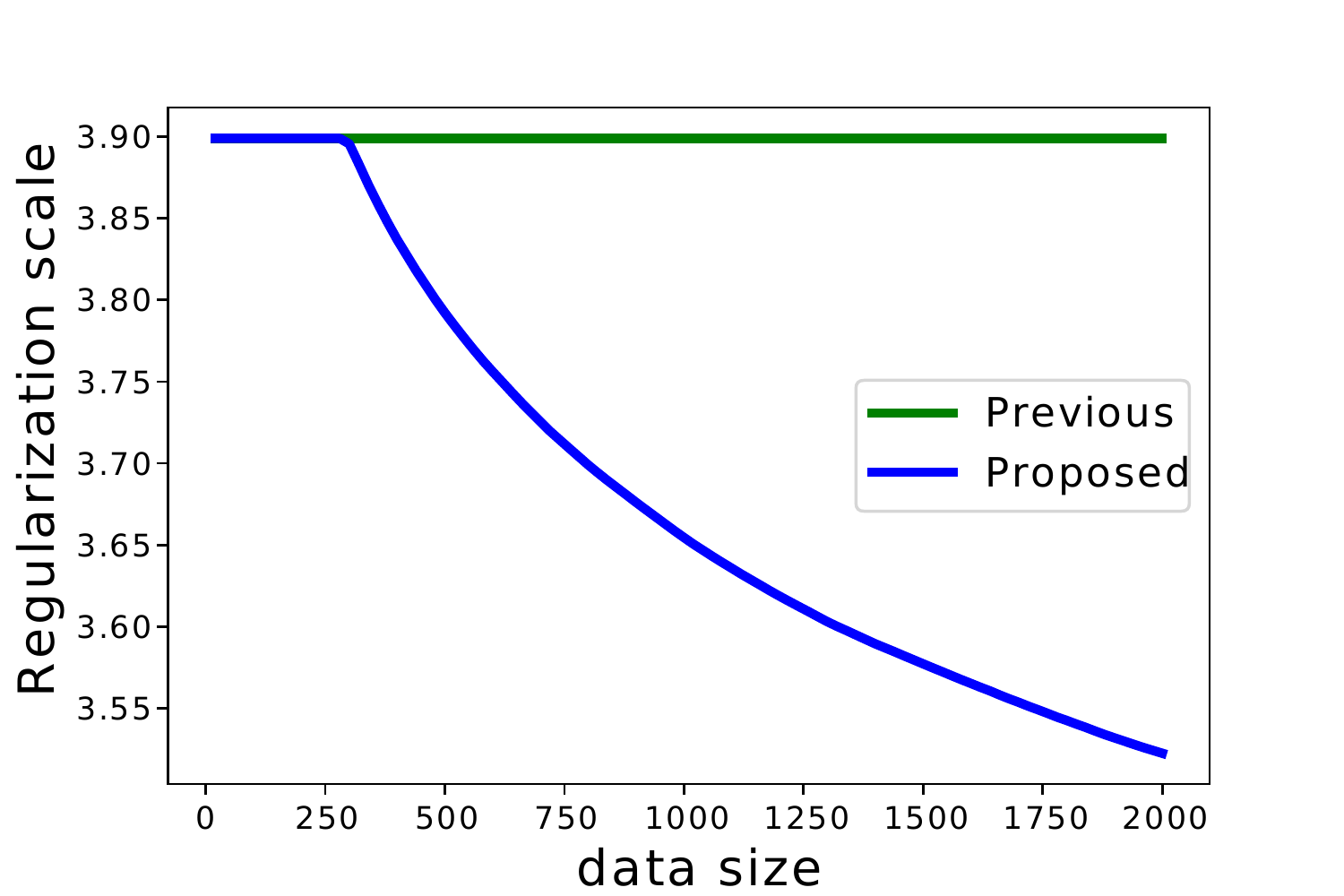}
			\caption{Decrease of regularization scale. The horizontal line shows the number of samples, and the vertical line shows the regularization scale. Respective green and blue lines show the scale of VBR and HSR.}\label{figRegScale}
		\end{center}
	\end{minipage}
\end{figure*}

\subsection{Experimental results}
Figure~\ref{figGenError} and~\ref{figRegScale} show the result of the experiments.
Figure~\ref{figGenError} plots the generalization error $L^*(h(X^n)) - L^*_{\min}$
of ERM (red), VBR (green), and HSR (blue).
We observe that, with small sample size $n \leq 1000$,
regularized solutions (VBR and HSR) showed smaller generalization error
than non-regularized solution (ERM).
This indicates that regularization can improve generalization error
by preventing over-fitting on risky hypothesis with high variance.
With large sample size $n \geq 1000$, on the other hand,
ERM showed smaller error than VBR.
This indicates that the regularization scale of VBR is over-conservative for large $n$,
which over-conservativeness prevents faster convergence.
The proposed algorithm (HSR)
shows the best performance with wide range of sample size $n \leq 1500$,
and, even with large sample size $n \geq 1500$,
in comparison with VBR,
HSR showed competitive performance to ERM.
This performance can be explained by Figure 2,
which plots the regularization scale of VBR and HSR against the sample size $n$.
Owing to the hypothesis space reduction mechanism,
once sample number get large enough $n \geq 300$,
HSR can automatically reduce the regularization scale for avoiding over-conservativeness.
Thus, the proposed algorithm achieve both 
stability of regularization in small $n$
and fast convergence of non-regularization in large $n$
at the same time.



\bibliographystyle{plain}


\normalsize
\appendix

\section{Proofs}
\subsection{Proof of Proposition~\ref{propVBRCondition}}
We first introduce the following concentration inequalities.
\begin{lem}[{\cite[Theorem 10]{maurer2009empirical}}] \label{lemVar}
	Then with probability at least $1-\delta$, it holds that
	\begin{align*}
		|r^*(h) - r_n(h; X^n)| \leq \frac{2\log 2/\delta}{n-1}
	\end{align*}
\end{lem}
\begin{lem}[Bennett's inequality]\label{lemBennett}
	Let $Z,Z_1,\dots,Z_n$ be i.i.d. random variables with values in $[0,1]$ and let $\delta>0$.
	Then with probability at least $1-\delta$, it holds that
	\begin{align*}
		\left| \rE[Z] - \frac{1}{n}\sum_{i=1}^n Z_i \right| \leq \sqrt{\frac{2 \Var[Z] \log2/\delta }{n}} + \frac{ \log 2/\delta}{3n}.
	\end{align*}
\end{lem}
\begin{lem}[Hoeffding's inequality]\label{lemHoeffding}
	Let $Z,Z_1,\dots,Z_n$ be i.i.d. random variables with values in $[0,1]$ and let $\delta>0$.
	Then with probability at least $1-\delta$, it holds that
	\begin{align*}
		Check\\
		\left| \rE[Z] - \frac{1}{n}\sum_{i=1}^n Z_i \right| \leq \sqrt{\frac{\log1/\delta}{2n}}.
	\end{align*}
\end{lem}
Proposition~\ref{propVBRCondition} can then be proven as follows.
\begin{proof}[Proof of Proposition~\ref{propVBRCondition}]
	(i) We first observe that, for any $h_1, h_2 \in \cH$ and $x^n \in \cX^n$, it holds that
	\begin{align*}
		&r_n(h_1; x^n) - r_n(h_2; x^n)\\
		&= \sqrt{\frac{1}{n(n-1)} \sum_{1 \leq i<j \leq n} (\ell(h_1,x_i) - \ell(h_1,x_j))^2 } - \sqrt{\frac{1}{n(n-1)} \sum_{1 \leq i<j \leq n} (\ell(h_2,x_i) - \ell(h_2,x_j))^2 }  \\
		&\leq \sqrt{\left| \frac{1}{n(n-1)} \sum_{1 \leq i<j \leq n}  \left((\ell(h_1,x_i) - \ell(h_1,x_j))^2 -  (\ell(h_2,x_i) - \ell(h_2,x_j))^2 \right)  \right| }\\
		&\leq \sqrt{\left| \frac{2}{n(n-1)} \sum_{1 \leq i<j \leq n}  \left| |\ell(h_1,x_i) - \ell(h_1,x_j)| -  |\ell(h_2,x_i) - \ell(h_2,x_j)|  \right|^2  \right| }\\
		&\leq \sqrt{\left| \frac{2}{n(n-1)} \sum_{1 \leq i<j \leq n}  (2 c_{\ell} \|h_1 - h_2 \|)^2   \right| }\\
		&\leq 2 \sqrt{2}c_{\ell} \|h_1 - h_2 \|  .
	\end{align*}
	Similarly, it holds that
	\begin{align}
		r^{V*}(h_1) - r^{V*}(h_2) \leq 2\sqrt{2}c_{\ell} \|h_1 - h_2 \| \label{ineqDiffVar}
	\end{align}
	By the definition of the covering number,
	there exists a finite subset $\cH_0 \subseteq \cH$ such that $|\cH_0| \leq \cN(1/n,\cH)$ and,
	for any $h \in \cH$, there exists $h_0 \in \cH$ such that $\| h - h_0 \| \leq 1/n$.
	By Lemma~\ref{lemVar} and the union bound, with probability $1-\delta/n$, the following holds for all $h_0 \in \cH_0$:
	\begin{align*}
		|r^*(h) - r_n(h; X^n)| \leq \sqrt{\frac{2\log 2n \cN(1/n,\cH)/\delta}{n-1} }.
	\end{align*}
	For any $h \in \cH$, there exists $h_0 \in \cH_0$ with $\| h - h_0 \| \leq 1/n$, and thus
	\begin{align*}
		|r^*(h) - r_n(h; X^n)| &\leq |r^*(h_0) - r_n(h_0; X^n)| + 4\sqrt{2}c_{\ell}\| h - h_0 \|\\
		&\leq \sqrt{\frac{3\log (2n \cN(1/n,\cH)/\delta)}{n} } + \frac{4\sqrt{2} c_{\ell}}{n}
	\end{align*}
	The second inequality holds since $1/ (n-1) \leq 3/(2n)$ for $n \geq 6$.
	
	(ii) It is trivial since $\ell$ takes value in $[0,1]$.
	
	(iii) For any $h_1, h_2 \in \cH$ and $x^n \in \cX^n$, it holds that
	\begin{align}
		|L(h_1; x^n) - L(h_2; x^n)| \leq c_{\ell} \|h_1 - h_2 \|,\label{ineqii1} \\
		|L^*(h_1) - L^*(h_2)| \leq c_{\ell} \|h_1 - h_2 \|.\label{ineqii2} 
	\end{align}
	With $\cH_0$ defined above, by Lemma~\ref{lemBennett} and the union bound,
	with a probability of at least $1-\delta/n$, the following holds for all $h_0 \in \cH_0$:
	\begin{align}
		|L^*(h_0 ) - L(h_0 ;X^n)| \leq \sqrt{\frac{2 \log ( 2n\cN(1/n,\cH)/\delta ) }{n}} r^*(h) + \frac{\log ( 2n \cN(1/n,\cH)/\delta )}{3n} \label{ineqii3} 
	\end{align}
	For any $h \in \cH$, there exists $h_0 \in \cH_0$ with $\| h - h_0 \| \leq 1/n$, and thus
	\begin{align*}
		&|L^*(h ) - L(h ;X^n)|\\
		&\leq |L^*(h_0 ) - L(h_0 ;X^n)| + |L(h_0; x^n) - L(h; x^n)| + |L^*(h_0) - L^*(h)| \\ 
		&\leq |L^*(h_0 ) - L(h_0 ;X^n)|  +  2c_{\ell} \|h - h_0 \|\\
		&\leq \sqrt{\frac{2 \log ( 2n \cN(1/n,\cH)/\delta )}{n}} r^*(h_0) + \frac{6c_{\ell} + \log ( 2n \cN(1/n,\cH)/\delta )}{3n}\\
		&\leq \sqrt{\frac{2 \log( 2n \cN(1/n,\cH)/\delta )}{n}} \left(r^*(h) + \frac{2\sqrt{2 } c_{\ell}}{n} \right) +  \frac{6c_{\ell} +  \log( 2n \cN(1/n,\cH)/\delta )}{3n}\\
		&= \sqrt{\frac{2 \log ( 2n \cN(1/n,\cH)/\delta ) }{n}} r^*(h) +  \frac{6c_{\ell} +\log ( 2n \cN(1/n,\cH)/\delta )}{3n} + \frac{ 4 c_{\ell} \sqrt{ \log ( 2n \cN(1/n,\cH)/\delta )}}{n^{3/2}}\\
		&\leq \sqrt{\frac{2 \log (2n \cN(1/n,\cH)/\delta) }{n}} r^*(h) +  \frac{(4 c_{\ell}+1) \log (2n \cN(1/n,\cH)/\delta)  }{n}.
	\end{align*}
	The second inequality follows from \eqref{ineqii1} and \eqref{ineqii2},
	and the third inequality follows from \eqref{ineqii3} and $\| h - h_0 \| \leq 1/n$.
	The forth ineuqlity follows from \eqref{ineqDiffVar}.
	The last inequality holds since $n^{1/2} \geq 2$ and $1 \leq \sqrt{\log 2n \cN(1/n,\cH)/\delta} \leq \log (2n \cN(1/n,\cH)/\delta)$.
	
	(iv) By the definition of the covering number,
	there exists a finite subset $\cF_0 \subseteq \cF$ such that $|\cF_0| \leq \cN(\varepsilon_n,\cF)$ and,
	for any $h \in \cF$, there exists $h_0 \in \cF$ such that $\| h - h_0 \| \leq \varepsilon_n$.
	For $h \in \cF$, $h_0 \in \cF_0$ and $\Delta h := h - h_0$, by Taylor's theorem,
	there exists $q \in [0,1] $ satisfying
	\begin{align*}
		|L^*(h) - L^*(h_0)| = \left| \rE \left[ \nabla \ell(\cdot, X)(h_0)^{\top} \right] \Delta h +     \Delta h^{\top}  \rE \left[\frac{\nabla^2 \ell(\cdot, X)(h_0 + q \Delta h)}{2}  \right]\Delta h   \right| 
	\end{align*}
	Then we have
	\begin{align}
		&\left| \rE \left[ \nabla \ell(\cdot, X)(h_0)^{\top} \right] \Delta h \right| - \frac{p_2^* }{2} \| h - h_0 \|^2 \nonumber \\
		&\leq \left| \rE \left[ \nabla \ell(\cdot, X)(h_0)^{\top} \right] \Delta h \right| -  \left|  \Delta h^{\top}  \rE \left[\frac{\nabla^2 \ell(\cdot, X)(h_0 + q \Delta h)}{2}  \right]\Delta h   \right| \nonumber \\
		&\leq |L^*(h) - L^*(h_0)| \leq c_{L^*}(\cF) \|h - h_0  \| . \label{ineqiv2}
	\end{align}
	Also for any $x^n \in \cX^n$, by Taylor's theorem, there exists $q' \in [0,1] $ satisfying
	\begin{align}
		|L(h; x^n) - L^*(h_0; x^n)| &= \left| \frac{1}{n} \sum_{i=1}^n  \nabla \ell(\cdot, X_i)(h_0)^{\top} \Delta h + \frac{1}{n} \sum_{i=1}^n   \Delta h^{\top}  \frac{\nabla^2  \ell(\cdot, X_i)(h_0 + q' \Delta h)}{2} \Delta h   \right| \nonumber\\
		&\leq \left| \frac{1}{n} \sum_{i=1}^n  \nabla \ell(\cdot, X_i)(h_0)^{\top} \Delta h \right|  + \frac{p_2^* }{2} \| h - h_0 \|^2\label{ineqiv3}
	\end{align}
	
	By Lemma~\ref{lemBennett} and the union bound,
	with a probability of $1- (n-3)\delta/n$, 
	the following holds for all $h_0 \in \cH_0$:
	\begin{align}
		|L^*(h_0 ) - L(h_0 ;X^n)| &\leq \sqrt{\frac{2 \log 2n\cN(\varepsilon_n,\cF)/(n-3)\delta }{n}} r^*(h_0) + \frac{\log (2n\cN(\varepsilon_n,\cF)/(n-3)\delta  )}{3n} \nonumber \\
		&\leq \sqrt{\frac{2 \log 2n\cN(\varepsilon_n,\cF)/(n-3)\delta }{n}} r^*(h_0) + \frac{\log (4\cN(\varepsilon_n,\cF)/\delta) }{3n} \label{probiv1}
	\end{align}
	The last inequality holds since $n \geq 6$.
	By Lemma~\ref{lemHoeffding}, with a probability of $1-\delta/n$, the following holds for all $h_0 \in \cF_0$ and $j=1,2,\dots,d$:
	\begin{align*}
		\left| \frac{1}{n}  \sum_{i=1}^n \frac{ \partial \ell(\cdot , X_i)}{ \partial h_j} - \rE \left[ \frac{ \partial \ell(\cdot , X)}{ \partial h_j} \right] \right| \leq \sqrt{ \frac{ 2p_1^{*2} \log ( 2 d n \cN(\varepsilon_n,\cF)/\delta )}{n}}, 
	\end{align*}
	and thus
	\begin{align*}
		\left\| \frac{1}{n}  \sum_{i=1}^n\nabla \ell(\cdot , X_i) - \rE \left[ \nabla \ell(\cdot , X) \right] \right\| \leq \sqrt{ \frac{ 2p_1^{*2}d^2 \log ( 2 d n \cN(\varepsilon_n,\cF)/\delta )}{n}}. 
	\end{align*}
	This inequality together with \eqref{ineqiv2} and \eqref{ineqiv3} implies that
	\begin{align}
		&|L(h_0 ;X^n) - L(h ;X^n)| \nonumber \\
		&\leq \left| \frac{1}{n} \sum_{i=1}^n  \nabla \ell(\cdot, X_i)(h_0)^{\top} \Delta h \right|  + \frac{p_2^* }{2} \| h - h_0 \|^2\nonumber \\
		&\leq \left| \rE \left[ \nabla \ell(\cdot , X)^{\top} \right] \Delta h \right| +   \left| \left(\frac{1}{n} \sum_{i=1}^n  \nabla \ell(\cdot, X_i)(h_0)^{\top}  - \rE \left[\nabla \ell(\cdot , X)(h_0)^{\top}\right] \right)  \Delta h \right|   + \frac{p_2^* }{2} \| h - h_0 \|^2\nonumber \\
		&\leq c_{L^*}(\cF) \|h - h_0  \| + \left\| \left(\frac{1}{n} \sum_{i=1}^n  \nabla \ell(\cdot, X_i)(h_0)^{\top}  - \rE \left[\nabla \ell(\cdot , X)(h_0)^{\top}\right] \right)\right\| \|h - h_0  \| + p_2^*  \| h - h_0 \|^2 \nonumber \\
		&\leq    c_{L^*}(\cF) \|h - h_0  \| +  \sqrt{ \frac{ 2p_1^{*2}d^2 \log ( 2 d n \cN(\varepsilon_n,\cF)/\delta )}{n}}  \|h - h_0  \| + p_2^*  \| h - h_0 \|^2. \label{probiv2}
	\end{align}
	
	Both \eqref{probiv1} and \eqref{probiv2} holds with probaiblity at least $1 - (n-2)/(\delta n)$, by the union bound.
	Then, for any $h \in \cF$, there exists $h_0 \in \cF_0$ such that $\|h - h_0 \| \leq \varepsilon_n$, and thus
	\begin{align*}
		&|L^*(h ) - L(h ;X^n)| \\
		&\leq |L^*(h_0 ) - L(h_0 ;X^n)| + |L^*(h_0 ) - L^*(h)| + |L(h_0 ;X^n) - L(h ;X^n)|\\
		&\leq \left(\sqrt{\frac{2 \log 2n\cN(\varepsilon_n,\cF)/(n-3)\delta }{n}} r^*(h_0) + \frac{\log (4\cN(\varepsilon_n,\cF)/\delta) }{3n} \right) \\
		&\quad + 2c_{L^*}(\cF) \| h - h_0 \|  + \sqrt{ \frac{ 2p_1^{*2}d^2 \log ( 2 d n \cN(\varepsilon_n,\cF)/\delta )}{n}}  \|h - h_0  \| + p_2^*  \| h - h_0 \|^2. \\
		&\leq  \sqrt{\frac{2 \log 2n\cN(\varepsilon_n,\cF)/ (n-3)\delta }{n}}( r^*(h) + 2\sqrt{2} c_{\ell} \varepsilon_n  ) + \frac{\log (4\cN(\varepsilon_n,\cF)/\delta ) }{3n}  \\
		&\quad + \left(2c(\cF) + \sqrt{ \frac{ 2p_1^{*2} d^2 \log ( 2 d n \cN(\varepsilon_n,\cF)/\delta)}{n}} \right) \varepsilon_n  + p_2^* \varepsilon_n^2 \\
		&= \sqrt{\frac{2 \log 2 n\cN(\varepsilon_n,\cF)/(n-3) \delta }{n}}r^*(h) + \frac{\log (4\cN(\varepsilon_n,\cF)/\delta ) }{3n}  \\
		&\quad + \left(2c(\cF) + \sqrt{ \frac{ 2p_1^{*2} d^2 \log 2 d n \cN(\varepsilon_n,\cF)/\delta}{n}} + \sqrt{\frac{8 \log (4\cN(\varepsilon_n,\cF)/\delta ) }{n}} c_{\ell}\right) \varepsilon_n  + p_2^* \varepsilon^2 \\
		&\leq \sqrt{\frac{2 \log (2 n\cN(\varepsilon_n,\cF)/(n-3) \delta )}{n}}r^*(h) + \frac{2 c(\cF) \log^{1/4} (n/\delta)}{n^{1/4+1/d}} + \frac{p_2^* \log^{1/2}(n/\delta)}{n^{1/2 + 2/d}} \\
		&\quad + \frac{4\sqrt{ p_1^{*2} d^2 + c_{\ell}^2} \log^{3/4} (2dn \cN(\varepsilon_n,\cF)/\delta)  }{n^{3/4 + 1/d}} + \frac{\log (4\cN(\varepsilon_n,\cF)/\delta)}{3n}
	\end{align*}
	The second inequality follows from \eqref{probiv1}, \eqref{probiv2}, and the definition of $c_{L^*}(\cF)$.
	The third inequality follows from \eqref{ineqDiffVar}.
\end{proof}

\subsection{Proof of Theorem~\ref{thmMain}}

\begin{proof}[Proof of Theorem~\ref{thmMain}]
	Let $h^* \in \cH^*$ denote the true optimum hypothesis 
	satisfying $h^* = {\argmin}_{h \in \cH^*} r^*(h)$.
	Let us define $\underline{\cF}$ by
	\begin{align*}
		\underline{\cF} := \left\{h \in  \cH \left| L^*(h) - L^*_{\min} \leq \frac{ \alpha_n (\max\{ r^*(h) , r^*_{\cH^*}\} + r^*_{\cH^*} + 3\Delta_n)  }{\sqrt{n}}  + 3\beta_n \right\} \right..
	\end{align*}
	By the uniform bound, for $X^n \sim \cP^n$, \eqref{dfEmpReg}, \eqref{dfUnif}, and \eqref{dfSpatial} for $\underline{\cF}$ hold at the same time with a probability at least $1-\delta$.
	Thus it is enough to show that: if $x^n \in \cX^n$ satisfies 
	\begin{align}
		|r^*(h) - r_n(h;x^n)| &\leq \Delta_n ,\\
		|L^*(h) - L(h; x^n) | &\leq \frac{\alpha_n r^*(h) }{\sqrt{n}} + \beta_n, \quad \forall h \in \cH, \\
		|L^*(h) - L(h; x^n) | &\leq  \frac{\mu_n(\underline{\cF}) r^*(h) }{\sqrt{n}} + \nu_n^*(\underline{\cF}), \quad \forall h \in \underline{\cF},
	\end{align}
	then it holds that
	\begin{align}
		L^*(h(x^n)) - L^*_{\min} \leq  \frac{2\mu_n(\overline{\cF})}{\sqrt{n}}(r^*_{\cH^*} + \Delta_n) + 2 \nu_n (\overline{\cF}).
	\end{align}
	
	We first prove 
	\begin{align}
		L^*(h^*) \leq v(x^n) +  \frac{\alpha_n \Delta_n }{\sqrt{n}} + \beta_n \leq L^*(h^*)  +  \frac{2 \alpha_n (r^*(h^*)  + \Delta_n ) }{\sqrt{n}} + 2\beta_n. \label{hatV}
	\end{align}
	Let $g(x^n) \in \cH$ be the optimal solution corresponds $v(x^n)$ in Line~\ref{alg1line1}, i.e., 
	\begin{align*}
		g(x^n) := \argmin_{h \in \cH} L(h; x_n) + \frac{\alpha_n}{\sqrt{n}} r_n(h ;x^n).
	\end{align*}
	Then \eqref{hatV} holds since
	\begin{align*}
		L^*(h^*)  \leq L^*(g(x^n)) &\leq L(g(x^n); x^n) + \frac{\alpha_n r^*(g(x^n)) }{\sqrt{n}} + \beta_n \\
		&\leq L(g(x^n); x^n) + \frac{\alpha_n (r_n(g(x^n); x^n) + \Delta_n )}{\sqrt{n}} + \beta_n \\
		&= v(x^n) + \frac{\alpha_n \Delta_n}{\sqrt{n}}  + \beta_n \\
		&\leq L(h^*; x^n) + \frac{\alpha_n (r^*(h^*) + 2\Delta_n ) }{\sqrt{n}} + \beta_n \\
		&\leq L^*(h^*) + \frac{2 \alpha_n (r^*(h^*) + \Delta_n )}{\sqrt{n}}  + 2\beta_n .
	\end{align*}
	
	We then prove 
	\begin{align}
		\max_{h \in \underline{\cF}} r_n(h;D_n) \leq u_n(\cG(x_n); x^n) \leq u_n^*( \overline{\cG}).
	\end{align}
	Since
	\begin{align*}
		L(h^*; x^n) &\leq L^* (h^*) + \frac{\alpha_n r^*(h^*)}{\sqrt{n}}  + \beta_n \\
		&\leq  v(x^n) + \frac{\alpha_n (r^*(h^*; x^n) +  2\Delta_n )}{\sqrt{n}}  + 2\beta_n,
	\end{align*}
	we have $h^* \in \cG(x^n)$.
	The first inequality holds since
	\begin{align*}
		\underline{\cF} &= \left\{h \in  \cH \left| L^*(h) \leq L^*(h^*) + \frac{ \alpha_n (\max\{ r^*(h) ,r^*(h^*)\} + r^*(h^*) + 3\Delta_n) }{\sqrt{n}}  + 3\beta_n  \right\} \right. \\
		&\subseteq \left\{h \in  \cH \left| L(h; x^n) \leq v(x^n) + \frac{ \alpha_n (\max\{ r_n(h; x^n) ,r_n(h^*; x^n)\} + r_n(h^* ; x^n) + r_n(h ; x^n) + 7\Delta_n) }{\sqrt{n}}  + 5\beta_n  \right\} \right..
	\end{align*}
	Since $h^* \in \underline{\cF}$, it holds that $\max_{h \in \underline{\cF}} r_n(h; x^n) \leq \max_{h \in \cG(x_n)} r_n(h; x^n) $. Then, by the definition of $u_n$, the first inequality holds.
	For the second inequality, observe that, for any $h \in \cG(x_n)$, we have
	\begin{align*}
		L^*(h) &\leq L(h;x^n) + \frac{\alpha_n r_n^*(h)}{\sqrt{n}}  + \beta_n\\
		&\leq v(x_n) + \frac{3 \alpha_n r(h;x^n) + \alpha_n r_n^*(h)  + 7\alpha_n \Delta_n}{\sqrt{n}}  + 6\beta_n\\
		&= v(x_n) + \frac{4 \alpha_n r^*(h) + 10\alpha_n \Delta_n }{\sqrt{n}} + 6\beta_n\\
		&\leq L^*(h^*) +  \frac{2 \alpha_n r^*(h^*) + \alpha_n\Delta_n}{\sqrt{n}}  + \beta_n + \frac{4 \alpha_n r^*(h) + 10\alpha_n \Delta_n }{\sqrt{n}} + 6\beta_n \\
		&\leq   L^*(h^*) +  \frac{6 \alpha_n r^*(h) + 11\alpha_n \Delta_n }{\sqrt{n}} + 7\beta_n,
	\end{align*}
	and thus $\cG(x_n) \subseteq \overline{\cG}$.
	Then the inequality holds by the definition of $u_n$ and $u_n^*$.
	
	Third, we prove 
	\begin{align}
		h(x^n) \in \underline{\cF} \subseteq \cF(x^n) \subseteq \overline{\cF}.
	\end{align}
	For the first inclusion, we have
	\begin{align*}
		L^*(h(x^n)) &\leq L(h(x^n); x^n) + \frac{ \mu_n (\cF(x^n))r^*(h(x^n)) + (\alpha_n - \mu_n(\cF(x^n))) r^*(h(x^n))}{ \sqrt{n}}  + \beta_n \\
		&\leq L(h(x^n); x^n) + \frac{ \mu_n (\cF(x^n)) (r_n(h(x^n) ; x^n) + \Delta_n) + (\alpha_n - \mu_n(\cF(x^n))) r^*(h(x^n)) }{ \sqrt{n}} + \beta_n\\
		&\leq L(h^*; x^n) + \frac{ \mu_n (\cF(x^n)) (r_n(h^* ; x^n) + \Delta_n) + (\alpha_n - \mu_n(\cF(x^n))) r^*(h(x^n)) }{ \sqrt{n}}  + \beta_n\\
		&\leq L^*(h^*) + \frac{\alpha_n r^*(h^*) }{\sqrt{n}} + \beta_n + \frac{ \mu_n (\cF(x^n)) (r^*(h^*) + 2\Delta_n) + (\alpha_n - \mu_n(\cF(x^n))) r^*(h(x^n))}{ \sqrt{n}}  + \beta_n \\
		&\leq L^*(h^*)  + \frac{\alpha_n (r^*(h^*) + \max\{r^*(h^*), r^*(h(x^n)) \} + 2\Delta_n )}{\sqrt{n}}  + 2\beta_n \\
		&\leq  v(x^n)  + \frac{\alpha_n (r^*(h^*) + \max\{r^*(h^*), r^*(h(x^n)) \} + 3\Delta_n )}{\sqrt{n}}  + 3\beta_n.
	\end{align*}
	and thus $h(x^n) \in \underline{\cF}$.
	For the second inclusion, for any $h \in \underline{\cF}$, then, we have
	\begin{align*}
		L(h; x^n) &\leq L^*(h) + \frac{\alpha_n r^*(h) }{\sqrt{n}} + \beta_n \\
		&\leq L^*(h^*) + \frac{ \alpha_n (\max\{ r^*(h) ,r^*(h^*)\} + r^*(h^*) + 3\Delta_n)  }{\sqrt{n}}  + 3\beta_n + \frac{\alpha_n r_n(h;x^n) + \alpha_n \Delta_n }{\sqrt{n}} + \beta_n \\
		&\leq v(x^n) +  \frac{\alpha_n \Delta_n }{\sqrt{n}} + \beta_n + \frac{ 2 \alpha_n u_n(\cG(x^n)) + 3\alpha_n \Delta_n)  }{\sqrt{n}}  + 3\beta_n + \frac{\alpha_n r_N(h;x^n) + \alpha_n \Delta_n}{\sqrt{n}}  + \beta_n \\
		&= v(x^N) +  \frac{ 3\alpha_n u_n(\cG(x^n)) + 5\alpha_n \Delta_n)  }{\sqrt{n}}  + 5\beta_n
	\end{align*}
	and thus $h \in \cF(x^n)$, which implies that $\underline{\cF} \subseteq \cF(x^n)$.
	The third inclusion, for any $h \in \cF(x^n)$, it holds that
	\begin{align*}
		L^*(h) &\leq L(h; x^n) + \frac{\alpha_n r^*(h) }{\sqrt{n}} + \beta_n\\
		&\leq v(x^n) +  \frac{ 3\alpha_n u_n(\cG(x^n)) + 5\alpha_n \Delta_n)  }{\sqrt{n}}  + 5\beta_n + \frac{\alpha_n r^*(h)}{\sqrt{n}}  + \beta_n\\
		&\leq L^*(h^*) +  \frac{ \alpha_n (r^*(h^*) + \Delta_n )}{\sqrt{n}} + \beta_n + \frac{ 3\alpha_n u_n(\cG(x^n)) + 5\alpha_n \Delta_n) }{\sqrt{n}}  + 5\beta_n + \frac{\alpha_n r^*(h)}{\sqrt{n}}  + \beta_n\\
		&= L^*(h^*) +  \frac{ \alpha_n (  r^*(h^*) + r^*(h) +  3 u_n(\cG(x^n)) + 6 \Delta_n )}{\sqrt{n}}  + 7\beta_n \\
		&\leq L^*(h^*) +  \frac{ \alpha_n (  r^*(h^*) + 4 u_n^*(\overline{\cG}) + 6 \Delta_n )}{\sqrt{n}}  + 7\beta_n,
	\end{align*}
	and thus $h \in \overline{\cF}$, which implies the desired inclusion.

	Finally, we have
	\begin{align*}
		L^*(h(x^n)) &\leq L^*(h(x^n); x^n) +  \frac{\mu_n(\underline{\cF}) r^*(h)}{\sqrt{n}}  + \nu_n^*(\underline{\cF}) \\
		&\leq L^*(h(x^n); x^n) +  \frac{\mu_n(\cF(x^n)) r_n(h; x^n) + \mu_n(\underline{\cF})\Delta_n }{\sqrt{n}} + \nu_n^*(\underline{\cF}) \\
		&\leq L^*(h^*; x^n) +  \frac{\mu_n(\cF(x^n)) r_n(h^*; x^n) + \mu_n(\underline{\cF})\Delta_n }{\sqrt{n}} + \nu_n^*(\underline{\cF}) \\
		&\leq L^*(h^*) + \frac{\mu_n(\underline{\cF}) r^*(h^*)}{\sqrt{n}}  + \nu_n^*(\underline{\cF}) +  \frac{\mu_n(\cF(x^n)) r_n(h^*; x^n) + \mu_n(\underline{\cF})\Delta_n}{\sqrt{n}}   + \nu_n^*(\underline{\cF}) \\
		&\leq L^*(h^*) + \frac{2 \mu_n(\overline{\cF})( r^*(h^*) +  \Delta_n)}{\sqrt{n}}  + 2\nu_n^*(\overline{\cF}) \\
		&=L^*_{\min} + \frac{2 \mu_n(\overline{\cF})( r^*_{\cH^*} +  \Delta_n)}{\sqrt{n}}  + 2\nu_n^*(\overline{\cF}) 
	\end{align*}
	The proof is complete.
\end{proof}

\begin{proof}[Proof of Corollary~\ref{cor1}]
	The statement holds since 
	\begin{align*}
		&\lim_{n \to \infty} \sqrt{n} \left( \frac{2\mu_n(\overline{\cF}_n)}{\sqrt{n}} (r^*_{\cH^*} + \Delta_n )  + 2 \nu_n^* (\overline{\cF}_n) - \frac{ 2\mu_n^*}{\sqrt{n}} r^*_{\cH^*}   \right) \\
		&= \lim_{n \to \infty} \left( 2(\mu_n(\cH(\xi_n))  - \mu_n^* )  r^*_{\cH^*} + \Delta_n  + 2 \sqrt{n}\nu_n^* (\cH(\xi_n)  \right) \\
		&= \lim_{n \to \infty}2 \mu_n^* r^*_{\cH^*}  \left( 1 - \frac{\mu_n(\cH(\xi_n))}{\mu_n^*}  \right) \\
		&\leq \lim_{\xi \to 0} \limsup_{n \to \infty }\mu_n^* r^*_{\cH^*}  \left( 1 - \frac{\mu_n(\cH(\xi))}{\mu_n^*}  \right) = 0.
	\end{align*}
\end{proof}

\subsection{Proof of Corollary~\ref{cor2}}
\begin{proof}[Proof of Corollary~\ref{cor2}]
	We first prove that 
	\begin{align*}
		c(\overline{\cF}_n) = O\left(\sqrt[4]{\frac{\log (n/\delta)}{n }} \right).
	\end{align*}
	If $n \geq n_0$, then
	\begin{align*}
		&\frac{ \alpha_n^V }{\sqrt{n}} (  5 u_n^{V*}(\overline{\cG}_n) + 6 \Delta^V_n ) + 7\beta^V_n \\
		&=\sqrt{ \frac{ 2 \log (2n \cN(1/n,\cH)/\delta) }{n}} \left(  5 + 6 \sqrt{\frac{3\log (2n \cN(1/n,\cH)/\delta)}{n} } + \frac{24\sqrt{2} c_{\ell}}{n} \right) \\
		&\quad+ \frac{7 (4c_{\ell} +1) \log (2n \cN(1/n,\cH)/\delta)  }{n}\\
		&= \sqrt{ \frac{ 50 \log (2n \cN(1/n,\cH)/\delta) }{n}} + \frac{ (6\sqrt{6} + 28c_{\ell} + 7) \log (2n \cN(1/n,\cH)/\delta)  }{n} \\
		&\quad+ \frac{48 c_{\ell} \sqrt{\log (2n \cN(1/n,\cH)/\delta)}  }{n^{3/2}} \\
		&\leq \sqrt{ \frac{ 50 \log (2n \cN(1/n,\cH)/\delta) }{n}} + \frac{(52 c_{\ell} + 25 )\log (2n \cN(1/n,\cH)/\delta)  }{n}.
	\end{align*}
	The last inequality holds since $\log (2n \cN(1/n,\cH)/\delta) \geq 1$, $\sqrt{\log (2n \cN(1/n,\cH)/\delta)} \leq \log (2n \cN(1/n,\cH)/\delta)$,
	and $n^{1/2} \geq 2$.
	Defining
	\begin{align*}
		\gamma_n :=  \sqrt{ \frac{ 50 \log (2n \cN(1/n,\cH)/\delta) }{n}} + \frac{(52 c_{\ell} + 25 )\log (2n \cN(1/n,\cH)/\delta)  }{n},
	\end{align*}
	it holds that $\gamma_n \leq \gamma_0$ and 
	$\overline{\cF} \subseteq \cH(\gamma_n)$.
	By \eqref{localQuad2}, it then holds that
	\begin{align*}
		c(\overline{\cF}_n) \leq b\sqrt{\gamma_n} = O\left(\sqrt[4]{\frac{\log (n/\delta)}{n }} \right).
	\end{align*}
	
	We next prove that 
	\begin{align*}
		\cN(\varepsilon_n,\overline{\cF}_n)) \leq \frac{cn}{2}.
	\end{align*}
	Observe that, by \eqref{localQuad1}, we have
	\begin{align*}
		\cH(\gamma_n) \subseteq \left\{ h \in \cH \left| \|h - h^* \| \leq \sqrt{\gamma_n / a} \right\} \right.
	\end{align*}
	Then we have
	\begin{align*}
		&\cN(\varepsilon_n,\overline{\cF}_n) \leq \cN(\varepsilon_n,\cH(\gamma_n)) \\
		&\leq \left( \frac{2 \sqrt{d\gamma_n / a}}{\varepsilon_n} \right)^d\\
		&=\left( 2 \sqrt{\frac{d}{a} \left( \sqrt{ \frac{ 50 \log (2n \cN(1/n,\cH)/\delta) }{n}} + \frac{(52 c_{\ell} + 25 )\log (2n \cN(1/n,\cH)/\delta)  }{n} \right) \frac{n^{1/2 + 2/d}}{\log^{1/2}(n/\delta)}  } \right)^d \\
		&= \left( 2 \sqrt{\frac{d}{a} \left( \sqrt{ \frac{50 \log (2n \cN(1/n,\cH)/\delta)}{\log(n/\delta) } }  + \frac{(52 c_{\ell} + 25 )\log (2n \cN(1/n,\cH)/\delta)  }{\sqrt{n\log(n/\delta)}} \right)   } \right)^d n \\
		&\leq \left( 2 \sqrt{\frac{dc' }{a}   } \right)^d n =  \frac{cn}{2}
	\end{align*}
	
	By Theorem~\ref{thmMain}, then, the following holds with a probability at least $1-\delta$:
	\begin{align*}
		&L^*(h(X^n)) - L^*_{\min} \\
		&\leq \frac{2\mu_n^V(\overline{\cF}_n)}{\sqrt{n}} (r^*_{\cH^*} + \Delta_n^V )  + 2 \nu_n^{V*}\\
		&\leq \sqrt{\frac{8 \log (2 n\cN(\varepsilon_n,\overline{\cF}_n))/(n-3) \delta )}{n}} \left(r^*_{\cH^*}  + \sqrt{\frac{3\log (2n \cN(1/n,\cH)/\delta)}{n} } + \frac{4\sqrt{2} c_{\ell}}{n} \right)  \\
		&\quad +  \frac{4 c(\overline{\cF}_n) \log^{1/4}(n/\delta)}{n^{1/4+1/d}} + \frac{2p_2 \log^{1/2}(n/\delta) }{n^{1/2 + 2/d}} \\
		&\quad  + \frac{8\sqrt{ p_1^2d + c_{\ell}^2} \log^{3/4} (2dn \cN(\varepsilon_n,\overline{\cF}_n)/\delta)  }{n^{3/4 + 1/d}} + \frac{2\log (4\cN(\varepsilon_n,\overline{\cF}_n)/\delta )}{3n} \\
		&= \sqrt{\frac{8 \log (2 n\cN(\varepsilon_n,\overline{\cF}_n))/(n-3) \delta )}{n}} r^*_{\cH^*}
		+ \frac{4 c(\overline{\cF}_n) \log^{1/4}(n/\delta)}{n^{1/4+1/d}} + O\left( \sqrt{ \frac{\log (n / \delta)}{n^{1 + 4/d}} } \right).\\
		&= \sqrt{ \frac{8 \log (cn / \delta )}{n}} r^*_{\cH^*}    + O\left( \sqrt{ \frac{\log (n / \delta)}{n^{1 + 2/d}} } \right).
	\end{align*}
\end{proof}

\end{document}